\documentclass[sn-mathphys]{sn-jnl}

\jyear{2022}%

\theoremstyle{thmstyleone}%
\newtheorem{theorem}{Theorem}%

\theoremstyle{thmstyletwo}%

\theoremstyle{thmstylethree}%

\usepackage{amssymb,bm,amsmath,url,color,lineno}

\begin{document}

\title[Feature subset selection for kernel SVM classification]{Feature subset selection for kernel SVM classification via mixed-integer optimization}

\author*[1,2]{\fnm{Ryuta} \sur{Tamura}}\email{r.tamura.cbc@gmail.com}

\author[3]{\fnm{Yuichi} \sur{Takano}}

\author[4]{\fnm{Ryuhei} \sur{Miyashiro}}

\affil[1]{\orgdiv{Graduate School of Engineering}, \orgname{Tokyo University of Agriculture and
Technology}, \orgaddress{\street{2-24-16 Naka-cho}, \city{Koganei-shi}, \postcode{184-8588}, \state{Tokyo}, \country{Japan}}}

\affil[2]{
\orgname{October Sky Co., Ltd.}, \orgaddress{\street{Zelkova Bldg., 1-25-12 Fuchu-cho}, \city{Fuchu-shi}, \postcode{183-0055}, \state{Tokyo}, \country{Japan}}}

\affil[3]{\orgdiv{Faculty of Engineering, Information and Systems}, \orgname{University of Tsukuba}, \orgaddress{\street{1-1-1 Tennodai}, \city{Tsukuba-shi}, \postcode{305-8573}, \state{Ibaraki}, \country{Japan}}}

\affil[4]{\orgdiv{Institute of Engineering}, \orgname{Tokyo University of Agriculture and
Technology}, \orgaddress{\street{2-24-16 Naka-cho}, \city{Koganei-shi}, \postcode{184-8588}, \state{Tokyo}, \country{Japan}}}

\abstract{%
We study the mixed-integer optimization (MIO) approach to feature subset selection in nonlinear kernel support vector machines (SVMs) for binary classification. 
First proposed for linear regression in the 1970s, this approach has recently moved into the spotlight with advances in optimization algorithms and computer hardware. 
The goal of this paper is to establish an MIO approach for selecting the best subset of features for kernel SVM classification. 
To measure the performance of subset selection, we use the kernel--target alignment, which is the distance between the centroids of two response classes in a high-dimensional feature space. 
We propose a mixed-integer linear optimization (MILO) formulation based on the kernel--target alignment for feature subset selection, and this MILO problem can be solved to optimality using optimization software. 
We also derive a reduced version of the MILO problem to accelerate our MILO computations. 
Experimental results show good computational efficiency for our MILO formulation with the reduced problem. 
Moreover, our method can often outperform the linear-SVM-based MILO formulation and recursive feature elimination in prediction performance, especially when there are relatively few data instances. 
} 


\maketitle

\section{Introduction}\label{sec1}

\subsection{Overview and related work}
Support vector machines (SVMs) are a family of sophisticated pattern recognition methods based on optimal separating hyperplanes. 
This method was first devised for binary classification by Boser et~al.~\cite{BoGu92} in combination with the kernel method~\cite{AiBr64} for nonlinear data analyses. 
Since then, SVMs have attracted considerable attention in various scientific fields because of their solid theoretical foundations and high generalization ability~\cite{CeGa20,CrSh00,Va98}.  
Kernel methods have been extended to a variety of multivariate analyses (e.g., principal component analysis, cluster analysis, and outlier detection)~\cite{ScSm02,ShCh04}, and they have also been applied to dynamic portfolio selection~\cite{TaGo11,TaGo14}. 

Feature subset selection involves selecting a subset of relevant features used in machine learning models. 
Such selection helps to understand the causality between predictor features and response classes, and it reduces the data collection/storage costs and the computational load of training machine learning models. 
Moreover, the prediction performance can be improved because overfitting is mitigated by eliminating redundant features. 
Because of these benefits, algorithms for feature subset selection have been extensively studied~\cite{ChSa14,GuEl03,LiCh17,LiMo07}. 
These algorithms can be categorized into filter, wrapper, and embedded methods. 
Filter methods (e.g., Fisher score~\cite{GuGu08} and relief~\cite{CaSh07,KiRe92}) rank features according to evaluation criteria before the training phase. 
Wrapper methods (e.g., recursive feature elimination~\cite{GuGu08}) search for better subsets of features through repeated training of subset models. 
Embedded methods (e.g., $L_1$-regularized estimation~\cite{HaTi19}) provide a subset of features as a result of the training process. 

We address the mixed-integer optimization (MIO) approach to feature subset selection. 
First proposed for linear regression in the 1970s~\cite{ArDo81}, this approach has recently gained attention with advances in optimization algorithms and computer hardware~\cite{BeKi16b,CoSa14,HaTi20,KoYa20,UsRu16}. 
Compared with many heuristic optimization algorithms, the MIO approach has the advantage of selecting the best subset of features with respect to given criterion functions~\cite{MiTa15a,MiTa15b,TaMi20}.
MIO methods for feature subset selection have been extended to logistic regression~\cite{BeKi17,SaTa16}, ordinal regression~\cite{NaTa19,SaTa17}, count regression~\cite{SaKu21}, dimensionality reduction~\cite{BeBe19,WaTa21}, and elimination of multicollinearity~\cite{BeKi16a,BeLi20,TaKo17,TaKo19}.
MIO-based high-performance algorithms have also been designed for feature subset selection~\cite{BePa20,BePa21,DeHa21,HaMa20,HaMa21,KuTa20}. 

Several prior studies have dealt with feature subset selection in linear SVM classification. 
A typical approach involves approximating the $L_0$-regularization term (or the cardinality constraint) for subset selection by the concave exponential function~\cite{BrMa98}, the $L_1$-regularization term~\cite{BrMa98,WeEl03,Zhu03}, and convex relaxations~\cite{ChVa07,GhNa18}, and the $L_0$-regularization term can be handled more accurately by DC (difference of convex functions) algorithms~\cite{GaGo20,LeLe15}. 
Meanwhile, Maldonado et~al.~\cite{MaPe14} proposed exact MIO formulations for feature subset selection in linear SVM classification, and Labb{\'e} et~al.~\cite{LaMa19} applied a heuristic kernel search algorithm to the MIO problem; Lagrangian relaxation~\cite{GaGo17} and generalized Benders decomposition~\cite{Ay15} have also been used to handle large-scale MIO problems. 
However, since these algorithms are focused on linear classification, they cannot be applied to nonlinear classification based on the kernel method. 

The feature scaling approach has been studied intensively for feature subset selection in kernel SVM classification~\cite{ChVa02,GrCa02,MaLo18,MaWe11,WeMu00}. 
This approach introduces feature weights in a kernel function and updates them iteratively in the gradient descent direction. 
Other algorithms for feature subset selection in kernel SVM classification include the filter method based on local kernel gradients~\cite{HeBu00}, local search algorithms~\cite{MaKo07,MaWe09,Wa08}, and metaheuristic algorithms~\cite{HuWa06}. 
Several performance measures for kernel SVM classifiers have been used in feature subset selection. 
Chapelle et~al.~\cite{ChVa02} designed gradient descent methods for minimizing various performance bounds on generalization errors, Neumann et~al.~\cite{NeSc05} proposed DC algorithms to maximize the kernel--target alignment~\cite{CrKa06}, Wang~\cite{Wa08} considered the kernel class separability in feature subset selection, and Jim{\'e}nez-Cordero et~al.~\cite{JiMo21} used software for nonlinear optimization to obtain good-quality solutions to their min-max optimization problem. 
To our knowledge, however, no prior studies have developed an exact algorithm to compute the best subset of features in terms of a given performance measure for nonlinear kernel SVM classification. 

\subsection{Contribution}
The goal of this paper is to establish a practical MIO approach to selecting the best subset of features for nonlinear kernel SVM classification. 
In line with Neumann et~al.~\cite{NeSc05}, we use the kernel--target alignment~\cite{CrKa06} as an objective function for subset selection. 
Also known as the distance between two classes (DBTC)~\cite{SuZh10} in a high-dimensional feature space, the kernel--target alignment has many applications in various kernel-based machine learning algorithms~\cite{WaZh15}.  
First, we introduce an integer nonlinear optimization (INLO) formulation based on the kernel--target alignment for feature subset selection. 
Next, we reformulate the problem as a mixed-integer linear optimization (MILO) problem, which can be solved to optimality using optimization software. 
We also derive a reduced version of the MILO problem to accelerate our MILO computations. 

We assess the efficacy of our method through computational experiments using real-world and synthetic datasets. 
With the real-world datasets, our MILO formulations offer clear computational advantages over the INLO formulation. 
In addition, the problem reduction offers highly accelerated MILO computations. 
With the synthetic datasets, our method often outperforms the linear-SVM-based MILO formulation~\cite{MaPe14} and recursive feature elimination~\cite{GuGu08} in terms of accuracy for both classification and subset selection, especially when there are relatively few data instances.  

\subsection{Organization and notation}
In Section~\ref{sec2}, we review SVMs for binary classification and introduce the kernel--target alignment.
In Section~\ref{sec3}, we present our MIO formulations based on the kernel--target alignment for feature subset selection.  
We report the computational results in Section~\ref{sec4} and conclude in Section~\ref{sec5}. 

Throughout this paper, we denote the set of consecutive integers ranging from 1 to $n$ as $[n] := \{1,2,\ldots,n\}$.
We write a $p$-dimensional column vector as $\bm{x} := (x_j)_{j \in [p]} \in \mathbb{R}^p$, and an $m \times n$ matrix as $\bm{X} := (x_{ij})_{(i,j) \in [m]\times[n]} \in \mathbb{R}^{m \times n}$. 

\section{Support vector machines for binary classification}\label{sec2}

In this section, we review linear and kernel SVMs for binary classification~\cite{CeGa20,CrSh00,Va98}. 
We then introduce the kernel--target alignment~\cite{CrKa06,NeSc05} to be used for feature subset selection. 

\subsection{Linear SVM classification}\label{sec2-1}

We address the task of correctly assigning a binary class label $\hat{y} \in \{-1,+1\}$ to each $p$-dimensional feature vector $\bm{x} := (x_j)_{j \in [p]} \in \mathbb{R}^p$. 
We begin with the following linear classifier: 
\begin{align} \label{eq:lc}
f(\bm{x}) = \bm{w}^{\top} \bm{x} + b = \sum_{j=1}^p w_j x_j + b, 
\end{align}
where the bias term $b \in \mathbb{R}$ and the weight vector $\bm{w} := (w_j)_{j \in [p]} \in \mathbb{R}^p$ are parameters to be estimated. 
We predict a class label based on the sign of the classifier~\eqref{eq:lc} as
\begin{equation}\label{eq:clsrule}
\begin{cases}
f(\bm{x}) < 0 & \Rightarrow \quad \hat{y}=-1, \\
f(\bm{x}) > 0 & \Rightarrow \quad \hat{y}=+1. \\
\end{cases}
\end{equation}

Suppose that we are given a training dataset $\{(\bm{x}_i,y_i) \mid i \in [n]\}$ containing $n$ data instances, where $\bm{x}_i := (x_{ij})_{j \in [p]} \in \mathbb{R}^p$ and $y_i \in \{-1,+1\}$ for each $i \in [n]$. 
All data instances are correctly separated by Eq.~\eqref{eq:clsrule} using the linear classifier~\eqref{eq:lc} when 
\begin{align}\label{eq:sprb}
y_i f(\bm{x}_i) > 0 \quad (i \in [n]),
\end{align}
but this is usually impossible. 
For this reason, the linear separability condition~\eqref{eq:sprb} is relaxed by introducing a vector $\bm{\xi} := (\xi_i)_{i \in [n]} \in \mathbb{R}^n$ of classification errors. 

The soft-margin SVM minimizes the weighted sum of the $L_2$-regularization term $\|\bm{w}\|_2^2 = \sum_{j=1}^p w_j^2$ and classification errors as 
\begin{align}
\mathop{\mbox{minimize}}_{} & \quad \frac{1}{2} \|\bm{w}\|_2^2 + C \sum_{i=1}^n \xi_i \label{obj:CSVM} \\
\mbox{subject~to} 
& \quad y_i (\bm{w}^{\top} \bm{x}_i + b) \ge 1 - \xi_i \quad (i \in [n]), \label{con1:CSVM} \\
& \quad \xi_i \ge 0 \quad (i \in [n]), \label{con2:CSVM} \\
& \quad b \in \mathbb{R},~\bm{w} \in \mathbb{R}^p,~\bm{\xi} \in \mathbb{R}^n, \label{con3:CSVM}
\end{align}
where $C \in \mathbb{R}_{+}$ is a user-defined parameter representing the misclassification penalty. 
Throughout this paper, each optimization formulation lists decision variables on the last line (e.g., Eq.~\eqref{con3:CSVM}). 

The Lagrangian dual form of problem~\eqref{obj:CSVM}--\eqref{con3:CSVM} is expressed as 
\begin{align}
\mathop{\mbox{maximize}}_{} & \quad \sum_{i=1}^n \alpha_i - \frac{1}{2} \sum_{i=1}^n \sum_{h=1}^n \alpha_i \alpha_h y_i y_h \bm{x}_i^{\top} \bm{x}_h \label{obj:CSVM2} \\
\mbox{subject~to} 
& \quad \sum_{i=1}^n \alpha_i y_i = 0, \label{con1:CSVM2} \\
& \quad 0 \le \alpha_i \le C \quad (i \in [n]), \label{con2:CSVM2} \\
& \quad \bm{\alpha} \in \mathbb{R}^n, \label{con3:CSVM2}
\end{align}
where $\bm{\alpha} := (\alpha_i)_{i \in [n]} \in \mathbb{R}^n$ is a vector composed of Lagrange multipliers. 
Note that $\bm{w} = \sum_{i=1}^n \alpha_i y_i \bm{x}_i$ holds from the optimality condition; therefore, the linear classifier~\eqref{eq:lc} is rewritten as
\begin{align} \label{eq:lc2}
f(\bm{x}) = \bm{w}^{\top} \bm{x} + b = \sum_{i=1}^n \alpha_i y_i \bm{x}_i^{\top} \bm{x} + b. 
\end{align}

\subsection{Kernel SVM classification}\label{sec2-2}

We consider a feature map $\bm{\phi}:\mathbb{R}^p \to \mathcal{X}$, which nonlinearly transforms the original feature vector $\bm{x}$ into a high-dimensional feature vector $\bm{\phi}(\bm{x})$ in a feature space $\mathcal{X}$. 
A simple example with $\bm{x} = (x_1,x_2)^{\top}$ is given by
\begin{equation} \label{eq:exfm}
\bm{\phi}(\bm{x}) = (x_1, x_2, x_1^2, x_1 x_2, x_2^2)^{\top}. 
\end{equation}
We also define the kernel function
\begin{equation}\label{eq:kf}
k(\bm{x},\bm{x}') = \bm{\phi}(\bm{x})^{\top} \bm{\phi}(\bm{x}'),
\end{equation}
which is the inner product in a feature space. 
The associated kernel matrix is written as 
\begin{equation}\label{eq:km}
\bm{K} := \Bigl( k(\bm{x}_i,\bm{x}_h) \Bigr)_{(i,h) \in [n] \times [n]} \in \mathbb{R}^{n \times n}.
\end{equation}

We replace the original feature vector $\bm{x}$ with the high-dimensional feature vector $\bm{\phi}(\bm{x})$ in problem~\eqref{obj:CSVM2}--\eqref{con3:CSVM2} and the corresponding classifier~\eqref{eq:lc2}. 
The kernel SVM is then formulated as the following optimization problem: 
\begin{align}
\mathop{\mbox{maximize}}_{} & \quad \sum_{i=1}^n \alpha_i - \frac{1}{2} \sum_{i=1}^n \sum_{h=1}^n \alpha_i \alpha_h y_i y_h k(\bm{x}_i,\bm{x}_h) \label{obj:CSVM3} \\
\mbox{subject~to} 
& \quad \sum_{i=1}^n \alpha_i y_i = 0, \label{con1:CSVM3} \\
& \quad 0 \le \alpha_i \le C \quad (i \in [n]), \label{con2:CSVM3} \\
& \quad \bm{\alpha} \in \mathbb{R}^n, \label{con3:CSVM3}
\end{align}
where the associated nonlinear classifier is given by
\begin{align} \label{eq:nlc}
f(\bm{x}) = \sum_{i=1}^n \alpha_i y_i k(\bm{x}_i,\bm{x}) + b. 
\end{align}

We consider the Gaussian kernel function defined as
\begin{equation} \label{eq:gauss}
k(\bm{x},\bm{x}') 
= \exp\left( -\gamma \|\bm{x} - \bm{x}'\|_2^2 \right) 
= \exp\left( -\gamma \sum_{j=1}^p (x_j - x'_j)^2 \right),
\end{equation}
where $\gamma \in \mathbb{R}_{+}$ is a user-defined scaling parameter. 
It is known that the Gaussian kernel function~\eqref{eq:gauss} corresponds to the inner product in an infinite-dimensional feature space. 

Consequently, if the kernel function~\eqref{eq:kf} is computable (e.g., Eq.~\eqref{eq:gauss}), the nonlinear classifier~\eqref{eq:nlc} can be obtained from solving the Lagrangian dual problem~\eqref{obj:CSVM3}--\eqref{con3:CSVM3}. 
Called the kernel method, this procedure avoids having to define feature maps explicitly (e.g., Eq.~\eqref{eq:exfm}). 

\subsection{Kernel--target alignment}\label{sec2-3}

For the class label vector $\bm{y} := (y_i)_{i \in [n]} \in \{-1,+1\}^n$, the \emph{kernel--target alignment}~\cite{CrKa06,NeSc05} of the kernel matrix~\eqref{eq:km} is defined as 
\begin{equation} 
\hat{A}(\bm{K},\bm{y}\bm{y}^{\top}) 
:= \frac{\bm{K} \bullet (\bm{y}\bm{y}^{\top})}{n\| \bm{K}\|_{\mathrm{F}}} = \frac{\sum_{i=1}^n \sum_{h=1}^n y_i y_h k(\bm{x}_i, \bm{x}_h)}{n\sqrt{\sum_{i=1}^n \sum_{h=1}^n k(\bm{x}_i, \bm{x}_h)^2}}. \label{eq:kta1}
\end{equation}
Following Neumann et~al.~\cite{NeSc05}, we focus on only the numerator of Eq.~\eqref{eq:kta1} in view of the boundedness of the kernel function~\eqref{eq:gauss}.

The index set of data instances in each class $y \in \{-1,+1\}$ is denoted by
\[
N(y) := \{i \in [n] \mid y_i = y \}. 
\]
We also define a vector of class labels divided by each class size as 
\begin{equation}\label{eq:clslbl}
\bm{\psi} := (\psi_i)_{i \in [n]} := \left( \frac{y_i}{\lvert N(y_i) \rvert} \right)_{i \in [n]} \in \mathbb{R}^n.
\end{equation}
Then, it follows from Eqs.~\eqref{eq:kf} and \eqref{eq:clslbl} that the kernel--target alignment~\cite{NeSc05} for the class label vector $\bm{\psi}$ is expressed as 
\begin{align}
  & \sum_{i=1}^n \sum_{h=1}^n \psi_i \psi_h k(\bm{x}_i, \bm{x}_h) \notag \\
= & \left\| \frac{1}{\lvert N(-1) \rvert}\sum_{i \in N(-1)}\bm{\phi}(\bm{x}_i) - \frac{1}{\lvert N(+1) \rvert}\sum_{i \in N(+1)}\bm{\phi}(\bm{x}_i) \right\|_2^2. \label{eq:kta2}
\end{align}
This corresponds to the squared Euclidean distance between the centroids of negative and positive classes in a high-dimensional feature space $\mathcal{X}$; in other words, the class separability in a feature space can be measured by Eq.~\eqref{eq:kta2}. 
In Section~\ref{sec3}, we use the kernel--target alignment~\eqref{eq:kta2} as an objective function to be maximized for subset selection. 

\section{Mixed-integer optimization formulations for feature subset selection}\label{sec3}

In this section, we present our MIO formulations based on the kernel--target alignment~\eqref{eq:kta2} for feature subset selection.  
We also apply some problem reduction techniques to our MIO formulations. 

\subsection{Integer nonlinear optimization formulation}\label{sec3-1}

For subset selection, we assume that all features are standardized as
\begin{equation}\label{eq:std}
\sum_{i=1}^n x_{ij} = 0 \quad \mathrm{and} \quad \frac{\sum_{i=1}^n x_{ij}^2}{n} = 1
\end{equation}
for all $j \in [p]$.

Let $\bm{z} := (z_j)_{j \in [p]} \in \{0,1\}^p$ be a vector composed of binary decision variables for subset selection; namely, $z_j = 1$ if the $j$th feature is selected, and $z_j = 0$ otherwise. 
We then consider the following subset-based Gaussian kernel function:
\begin{equation}\label{eq:sgk}
k_{\bm{z}}(\bm{x}, \bm{x}') := \exp\left( -\gamma \sum_{j=1}^p z_j (x_j - x'_j)^2 \right).
\end{equation}

To select the best subset of features for kernel SVM classification, we maximize the kernel--target alignment~\eqref{eq:kta2} based on the subset-based Gaussian kernel function~\eqref{eq:sgk}. 
This problem is formulated as the following INLO problem:
\begin{align}
\mathop{\mbox{maximize}}_{} & \quad \sum_{i=1}^n \sum_{h=1}^n \psi_i \psi_h \exp\left( -\gamma \sum_{j=1}^p z_j (x_{ij} - x_{hj})^2 \right) \label{obj:INLO} \\
\mbox{subject~to} 
& \quad \sum_{j=1}^p z_j \le \theta, \label{con1:INLO} \\
& \quad \bm{z} \in \{0,1\}^p, \label{con2:INLO}
\end{align}
where $\theta \in [p]$ is a user-defined subset size parameter. 

However, a globally optimal solution to problem~\eqref{obj:INLO}--\eqref{con2:INLO} is very difficult to compute because its objective function~\eqref{obj:INLO} is a nonlinear nonconvex function.

\subsection{Mixed-integer linear optimization formulations}\label{sec3-2}

Theorem~\ref{thm1} states that the INLO problem~\eqref{obj:INLO}--\eqref{con2:INLO} can be reformulated as the following MILO problem:  
\begin{align}
\mathop{\mbox{maximize}}_{} & \quad \sum_{i=1}^n \sum_{h=1}^n \psi_i \psi_h e_{ih,p+1} \label{obj:MILO1} \\
\mbox{subject~to} 
& \quad \sum_{j=1}^p z_j \le \theta, \label{con1:MILO1} \\
& \quad e_{ih1} = 1 \quad (i \in [n], h\in [n]), \label{con2:MILO1} \\
& \quad z_j = 0 ~\Rightarrow~ e_{ih,j+1} = e_{ihj} \quad (i \in [n], h \in [n], j \in [p]), \label{con3:MILO1} \\
& \quad z_j = 1 ~\Rightarrow~ e_{ih,j+1} = \exp\left(-\gamma(x_{ij} - x_{hj})^2\right) \cdot e_{ihj} \notag \\
& \quad \qquad (i \in [n], h \in [n], j \in [p]), \label{con4:MILO1} \\
& \quad \bm{e} \in \mathbb{R}_{+}^{n \times n \times (p+1)},~\bm{z} \in \{0,1\}^p, \label{con5:MILO1}
\end{align}
where $\bm{e} := (e_{ihj})_{(i,h,j) \in [n] \times [n] \times [p+1]} \in \mathbb{R}_{+}^{n \times n \times (p+1)}$ is an array of auxiliary nonnegative decision variables. 
Here, Eqs.~\eqref{con3:MILO1}~and~\eqref{con4:MILO1} are logical implications, which can be imposed by using indicator constraints implemented in modern optimization software. 

\begin{theorem}[]\label{thm1}
Let $(\bm{e}^*,\bm{z}^*)$ be an optimal solution to the MILO problem~\eqref{obj:MILO1}--\eqref{con5:MILO1}. 
Then, $\bm{z}^*$ is also an optimal solution to the INLO problem~\eqref{obj:INLO}--\eqref{con2:INLO}. 
\end{theorem}
\begin{proof}
For each $\bm{z} \in \{0,1\}^p$, we can construct a feasible solution $(\bm{e},\bm{z})$ to problem~\eqref{obj:MILO1}--\eqref{con5:MILO1} by using Eqs.~\eqref{con3:MILO1}--\eqref{con4:MILO1} recursively from Eq.~\eqref{con2:MILO1}. 
Therefore, the feasible region of $\bm{z}$ in problem~\eqref{obj:MILO1}--\eqref{con5:MILO1} is the same as in problem~\eqref{obj:INLO}--\eqref{con2:INLO}. 

Consequently, it is necessary to prove only that, in Eqs.~\eqref{obj:INLO} and \eqref{obj:MILO1},
\begin{equation}\notag 
e_{ih,p+1} = \exp\left( -\gamma \sum_{j=1}^p z_j (x_{ij} - x_{hj})^2 \right) 
\end{equation}
for all $(i,h) \in [n] \times [n]$. 
Note that
\begin{align}
\exp\left( -\gamma z_j (x_{ij} - x_{hj})^2 \right) =
\begin{cases}
\exp(0) = 1 & \mathrm{if}~z_j = 0, \\ 
\exp\left( -\gamma (x_{ij} - x_{hj})^2 \right) & \mathrm{if}~z_j = 1.
\end{cases}\notag
\end{align}
Therefore, the constraints~\eqref{con3:MILO1}~and~\eqref{con4:MILO1} can be integrated into 
\[
e_{ih,j+1} = \exp\left(-\gamma z_j (x_{ij} - x_{hj})^2\right) \cdot e_{ihj}
\]
for all $j \in [p]$. 
By substituting this equation recursively, we obtain
\begin{align}
e_{ih,p+1} 
& = \prod_{j=1}^p \exp\left( -\gamma z_j (x_{ij} - x_{hj})^2 \right) \quad \because \mathrm{Eq.~\eqref{con2:MILO1}} \notag \\
& = \exp\left( -\gamma \sum_{j=1}^p z_j (x_{ij} - x_{hj})^2 \right), \notag 
\end{align}
which completes the proof.  
\end{proof}

Note also that problem~\eqref{obj:MILO1}--\eqref{con5:MILO1} can be equivalently rewritten without logical implications as 
\begin{align}
\mathop{\mbox{maximize}}_{} & \quad \sum_{i=1}^n \sum_{h=1}^n \psi_i \psi_h e_{ih,p+1} \label{obj:MILO2} \\
\mbox{subject~to} 
& \quad \sum_{j=1}^p z_j \le \theta, \label{con1:MILO2} \\
& \quad e_{ih1} = 1 \quad (i \in [n], h \in [n]), \label{con2:MILO2} \\
& \quad -M z_j \le e_{ih,j+1} - e_{ihj} \le M z_j \quad (i \in [n], h \in [n], j \in [p]), \label{con3:MILO2} \\
& \quad -M \cdot (1 - z_j) \le e_{ih,j+1} - \exp\left(-\gamma(x_{ij} - x_{hj})^2\right) \cdot e_{ihj} \notag \\ 
& \quad \qquad \le M \cdot (1 - z_j) \quad (i \in [n], h \in [n], j \in [p]), \label{con4:MILO2} \\
& \quad \bm{e} \in \mathbb{R}_{+}^{n \times n \times (p+1)},~\bm{z} \in \{0,1\}^p, \label{con5:MILO2}
\end{align}
where $M \in \mathbb{R}_+$ is a sufficiently large positive constant (e.g., $M = 1$ due to Eqs.~\eqref{con2:MILO1}--\eqref{con4:MILO1}).

\subsection{Problem reduction}\label{sec3-3}

Let us define the following index sets of instance pairs:
\begin{align}
H &:= \{(i,h) \in [n] \times [n] \mid i < h\}, \notag \\
H_{+} &:= \{(i,h) \in [n] \times [n] \mid i < h,~\psi_i \psi_h > 0\}, \notag \\
H_{-} &:= \{(i,h) \in [n] \times [n] \mid i < h,~\psi_i \psi_h < 0\}. \notag
\end{align}
Then, Theorem~\ref{thm2} proves that the MILO problem~\eqref{obj:MILO2}--\eqref{con5:MILO2} can be reduced to the following MILO problem:
\begin{align}
\mathop{\mbox{maximize}}_{} & \quad \sum_{(i,h) \in H} \psi_i \psi_h e_{ih,p+1} \label{obj:MILO3} \\
\mbox{subject~to} 
& \quad \sum_{j=1}^p z_j \le \theta, \label{con1:MILO3} \\
& \quad e_{ih1} = 1 \quad ((i,h) \in H), \label{con2:MILO3} \\
& \quad e_{ih,j+1} - e_{ihj} \le 0 \quad ((i,h) \in H_{+}, j \in [p]), \label{con3:MILO3} \\
& \quad -M_{ihj} z_j \le e_{ih,j+1} - e_{ihj} \quad ((i,h) \in H_{-}, j \in [p]), \label{con4:MILO3} \\
& \quad e_{ih,j+1} - \exp\left(-\gamma(x_{ij} - x_{hj})^2\right) \cdot e_{ihj} \le M_{ihj} \cdot (1 - z_j) \notag \\ 
& \quad \qquad ((i,h) \in H_{+}, j \in [p]), \label{con5:MILO3} \\
& \quad 0 \le e_{ih,j+1} - \exp\left(-\gamma(x_{ij} - x_{hj})^2\right) \cdot e_{ihj} \notag \\ 
& \quad \qquad ((i,h) \in H_{-}, j \in [p]), \label{con6:MILO3} \\
& \quad \bm{e} \in \mathbb{R}_{+}^{\lvert H \rvert \times (p+1)},~\bm{z} \in \{0,1\}^p, \label{con7:MILO3}
\end{align}
where 
\begin{equation}\label{eq:bigM}
M_{ihj} := 1 - \exp\left(-\gamma(x_{ij} - x_{hj})^2\right) \quad (i \in [n], h \in [n], j \in [p]).
\end{equation}

\begin{theorem}[]\label{thm2}
Let $(\bm{e}^*,\bm{z}^*)$ be an optimal solution to the reduced MILO problem~\eqref{obj:MILO3}--\eqref{con7:MILO3}. 
Then, $\bm{z}^*$ is also an optimal solution to the INLO problem~\eqref{obj:INLO}--\eqref{con2:INLO}. 
\end{theorem}
\begin{proof}
Because of Theorem~\ref{thm1}, it is necessary to prove only that problem~\eqref{obj:MILO2}--\eqref{con5:MILO2}, which is equivalent to problem~\eqref{obj:MILO1}--\eqref{con5:MILO1}, can be reformulated as problem~\eqref{obj:MILO3}--\eqref{con7:MILO3}. 

We begin by focusing on the objective function~\eqref{obj:MILO2}. 
Note that Eq.~\eqref{obj:INLO} can be decomposed as 
\begin{align}
  & \sum_{i=1}^n \sum_{h=1}^n \psi_i \psi_h \exp\left( -\gamma \sum_{j=1}^p z_j (x_{ij} - x_{hj})^2 \right) \notag \\
= & \sum_{i=1}^n \psi_i^2 + 2 \sum_{(i,h) \in H} \psi_i \psi_h \exp\left( -\gamma \sum_{j=1}^p z_j (x_{ij} - x_{hj})^2 \right). \notag
\end{align}
This implies that the objective function~\eqref{obj:MILO2} can be replaced with Eq.~\eqref{obj:MILO3}. 
Accordingly, the unnecessary decision variables (i.e., $e_{ihj}$ for $(i,h) \not\in H$) and the corresponding subset of constraints~\eqref{con2:MILO2}--\eqref{con4:MILO2} can be deleted from the problem. 

Next, we consider constraints~\eqref{con3:MILO2}~and~\eqref{con4:MILO2}. 
It is clear from Eqs.~\eqref{con2:MILO1}--\eqref{con4:MILO1} that 
\[
0 \le \exp\left(-\gamma(x_{ij} - x_{hj})^2\right) \cdot e_{ihj} \le e_{ih,j+1} \le e_{ihj} \le 1. 
\]
Therefore, it follows that
\begin{align}
& -M_{ihj} \le -\underbrace{\left(1 - \exp\left(-\gamma(x_{ij} - x_{hj})^2\right)\right)}_{M_{ihj}} \cdot e_{ihj} \le e_{ih,j+1} - e_{ihj} \le 0, \notag \\
& 0 \le e_{ih,j+1} - \exp\left(-\gamma(x_{ij} - x_{hj})^2\right) \cdot e_{ihj}
\le \underbrace{\left(1 - \exp\left(-\gamma(x_{ij} - x_{hj})^2\right)\right)}_{M_{ihj}} \cdot e_{ihj}
\le M_{ihj}. \notag
\end{align}
This implies that the feasible region remains the same even if the constraints~\eqref{con3:MILO2} and~\eqref{con4:MILO2} are tightened as
\begin{align}
& -M_{ihj} z_j \le e_{ih,j+1} - e_{ihj} \le 0, \label{eq:con3} \\
& 0 \le e_{ih,j+1} - \exp\left(-\gamma(x_{ij} - x_{hj})^2\right) \cdot e_{ihj} \le M_{ihj} \cdot (1 - z_j). \label{eq:con4}
\end{align} 

When $\psi_i \psi_h > 0$, the left inequalities in Eqs.~\eqref{eq:con3}~and~\eqref{eq:con4} are redundant because $e_{ih,j+1}$ is maximized by the objective function~\eqref{obj:MILO2}. 
Similarly, when $\psi_i \psi_h < 0$, the right inequalities in Eqs.~\eqref{eq:con3}~and~\eqref{eq:con4} are redundant. 
As a result, constraints~\eqref{con3:MILO3}--\eqref{con6:MILO3} are obtained, thus completing the proof. 
\end{proof}

We conclude this section by highlighting the differences between the MILO problem~\eqref{obj:MILO2}--\eqref{con5:MILO2} and its reduced version~\eqref{obj:MILO3}--\eqref{con7:MILO3}. 
The number of continuous decision variables is reduced from $(p+1)n^2$ (Eq.~\eqref{con5:MILO2}) to $(p+1)n(n-1)/2$ (Eq.~\eqref{con7:MILO3}), and the number of inequality constraints is reduced from $4pn^2$ (Eqs.~\eqref{con3:MILO2} and \eqref{con4:MILO2}) to $pn(n-1)$ (Eqs.~\eqref{con3:MILO3}--\eqref{con6:MILO3}). 
Also, the big-$M$ values are equal (e.g., $M=1$) in Eqs.~\eqref{con3:MILO2} and \eqref{con4:MILO2}, whereas they are set to the smaller values~\eqref{eq:bigM} in Eqs.~\eqref{con4:MILO3} and \eqref{con5:MILO3}. 

\section{Computational experiments}\label{sec4}
In this section, we report the results of computations to evaluate the efficacy of our method for feature subset selection in kernel SVM classification.
First, we confirm the computational efficiency of our MIO formulations using real-world datasets, and then we examine the prediction performance of our method for feature subset selection using synthetic datasets. 

All computations were performed on a Windows computer with two Intel Xeon E5-2620v4 CPUs (2.10\,GHz) and 128\,GB of memory using a single thread.

\subsection{Experimental design for real-world datasets}\label{sec4a}

We downloaded four real-world datasets for classification tasks from the UCI Machine Learning Repository~\cite{DuGr19}. 
Table~\ref{tab:DataSet} lists the datasets, where $n$ and $p$ are the numbers of data instances and candidate features, respectively. 
Categorical variables with two categories were treated as dummy variables, and those with more than two categories were transformed into sets of dummy variables. 
In the Zoo and Parkinsons datasets, the names of data instances were deleted. 
In the Hepatitis dataset, we removed four variables containing more than 10 missing values, and then data instances containing missing values. 
In the Soybean dataset, variables with the same value in all data instances were eliminated.
The Zoo and Soybean datasets have multiple response classes, so the positive label (i.e., $y_i = +1$) was given to classes 1 and 2 in the Zoo dataset and to classes D1 and D4 in the Soybean dataset, and the negative label (i.e., $y_i = -1$) was given to the other classes.  

\begin{table}[htb]
\begin{center}
\caption{Real-world datasets}
\begin{tabular}{lrrl}\toprule
            Name & $n$ & $p$ & Original dataset~\cite{DuGr19} \\ \midrule
Hepatitis    & 138 & 15  & Hepatitis \\
Zoo          & 101 & 16  & Zoo \\
Parkinsons   & 195 & 22  & Parkinsons \\
Soybean      &  47 & 45  & Soybean (Small) \\
\bottomrule
\end{tabular}
\label{tab:DataSet}
\end{center}
\end{table}

We compare the computational efficiency of the following MIO formulations for feature subset selection in kernel SVM classification:  
\begin{description}
\item{\textbf{INLO-K}:} INLO formulation~\eqref{obj:INLO}--\eqref{con2:INLO}; 
\item{\textbf{MILO-K}:} MILO formulation~\eqref{obj:MILO2}--\eqref{con5:MILO2} with $M=1$;
\item{\textbf{RMILO-K}:} reduced MILO formulation~\eqref{obj:MILO3}--\eqref{con7:MILO3}.
\end{description}
The MILO problems were solved using the optimization software IBM ILOG CPLEX 20.1.0.0~\cite{ib22}, where algorithms for solving relaxed subproblems on each node were set to the interior-point method instead of the dual simplex method.
To increase numerical stability, the big-$M$ values for RMILO-K were set as 
\begin{equation}\label{eq:bigM2}
M_{ihj} = \min\{1 - \exp\left(-\gamma(x_{ij} - x_{hj})^2\right) + 0.1, 1.0\} \quad (i \in [n], h \in [n], j \in [p]).
\end{equation} 
The INLO problem, which cannot be handled by CPLEX because of its nonlinear objective function, was solved by the optimization software Gurobi Optimizer 9.5.0~\cite{Gu20} using the general constraint {\tt EXP} function.

Many algorithms have been proposed for tuning SVM hyperparameters~\cite{WaFo21}.
Based on the sigest method~\cite{CaSi02}, we estimated an appropriate value of the scaling parameter $\gamma$ in the subset-based Gaussian kernel function~\eqref{eq:sgk} as follows:
\begin{equation}\label{eq:sigest}
\hat{\gamma} := \frac{1}{\mathrm{median~of~}\left\{(\theta/p) \cdot \sum_{j=1}^p (x_{ij} - x_{hj})^2 \mid (i,h) \in H \right\}}.
\end{equation}
We then set $\gamma = \beta \hat{\gamma}$ with the scaling factor $\beta \in \{0.25, 1.00, 4.00\}$. 

The following column labels are used in Tables~\ref{tab:RW_Hep}--\ref{tab:Syn_100_5}: 
\begin{description}
\item{\textbf{ObjVal}:} value of the objective function~\eqref{obj:INLO};
\item{\textbf{OptGap}:} absolute difference between lower and upper bounds on the optimal objective value divided by the lower bound;
\item{$\bm{\lvert \hat{S} \rvert}$:} subset size of selected features;
\item{\textbf{Time}:} computation time in seconds.
\end{description}
A computation was terminated if it did not complete within 10000\,s; in those cases, the best feasible solution found within 10000\,s was taken as the result.

\subsection{Results for real-world datasets}\label{sec4b}

Tables~\ref{tab:RW_Hep}--\ref{tab:RW_Soy} give the computational results of the three MIO formulations for the real-world datasets. 
First, we focus on the results for the Hepatitis dataset (Table~\ref{tab:RW_Hep}). 
The INLO formulation (INLO-K) always reached the time limit of 10000\,s, and therefore its ObjVal values were often very small. 
In contrast, our reduced MILO formulation (RMILO-K) solved all the problem instances completely within the time limit. 
Moreover, RMILO-K finished computations much sooner than did the original MILO formulation (MILO-K); for example, the computation times of MILO-K and RMILO-K for $(\theta,\beta)=(3,0.25)$ were 9498.3\,s and 716.6\,s, respectively. 

For the Zoo dataset (Table~\ref{tab:RW_Zoo}), RMILO-K was still much faster than the other formulations, whereas the differences in ObjVal among the three formulations were relatively small. 
For the Parkinsons dataset (Table~\ref{tab:RW_Park}), although the three formulations failed to finish computations within the time limit, RMILO-K often attained the largest ObjVal and smallest OptGap values. 
For the Soybean dataset (Table~\ref{tab:RW_Soy}), INLO-K and MILO-K often reached the time limit, whereas RMILO-K solved all the problem instances to optimality. 
These results show that our MILO formulations offer clear computational advantages over the INLO formulation, and the problem reduction offers highly accelerated MILO computations. 

Next, we examine how the two user-defined parameters $(\theta,\beta)$ affected the MILO computations. 
The computation time of RMILO-K was longer when $\theta$ was large; this is reasonable because the number of feasible subsets of features increases with $\theta$. 
Also, the computation time of RMILO-K was often longest with $\beta = 1.00$; this implies that solving MILO problems is computationally expensive when the scaling parameter $\gamma$ is tuned appropriately by Eq.~\eqref{eq:sigest}. 

\subsection{Experimental design for synthetic datasets}\label{sec4c}

We prepared synthetic datasets based on the MADELON dataset~\cite{GuGu04} from the NIPS 2003 Feature Selection Challenge. 
Specifically, we supposed that there were $\theta^*$ relevant features and $p - \theta^*$ irrelevant features. 
The relevant features were generated using the NDCC (normally distributed clusters on cubes) data generator~\cite{Th06}, which is designed to create datasets for nonlinear binary classification. 
The expansion factor ``exp'' is used in the NDCC data generator to stretch the covariance matrix of multivariate normal distributions; as the expansion factor increases, two classes overlap and become difficult to discriminate. 
The irrelevant features were drawn randomly from the standard normal distribution.
All these features were standardized as in Eq.~\eqref{eq:std}. 

We used $n$ data instances as a training dataset for each combination $(n,p,\mathrm{exp},\theta^*)$ of parameter values. 
For this training dataset, we selected a subset $\hat{S}$ of features. 
The accuracy of subset selection is measured by the F1 score, which is the harmonic average of $\mathrm{Recall} := \lvert S^* \cap \hat{S} \rvert/\lvert S^* \rvert$ and $\mathrm{Precision} :=  \lvert S^* \cap \hat{S} \rvert/\lvert \hat{S} \rvert$ as follows:  
\[
\mathbf{SetF1} := \frac{2 \cdot \mathrm{Recall} \cdot \mathrm{Precision}}{\mathrm{Recall} + \mathrm{Precision}}, 
\]
where $S^*$ is the set of relevant features. 

By means of the training dataset, we trained SVM classifiers with the selected subset $\hat{S}$ of features. 
We then evaluated the prediction performance by applying the trained classifier to a testing dataset consisting of sufficiently many data instances. 
Let $\hat{y}_i(\hat{S})$ be the class label predicted for the $i$th data instance.   
The classification accuracy for the testing dataset is calculated as 
\[
\mathbf{ClsAcc} := \frac{\lvert \{i \in \tilde{N} \mid y_i = \hat{y}_i(\hat{S})\} \rvert}{\lvert \tilde{N} \rvert},
\]
where $\tilde{N}$ is the index set of testing data instances. 
We repeated this process 10 times and give average values in Tables~\ref{tab:Syn_25_3}--\ref{tab:Syn_100_5}. 

We compare the prediction performance of the following methods for feature subset selection: 
\begin{description}
\item{\textbf{MILO-L}:} MILO formulation (MILP2~\cite{MaPe14}) for linear SVM classification;
\item{\textbf{RFE-K}:} recursive feature elimination~\cite{GuGu08} for kernel SVM classification;
\item{\textbf{RMILO-K}:} our reduced MILO formulation~\eqref{obj:MILO3}--\eqref{con7:MILO3} for kernel SVM classification.
\end{description}
The MILO formulation (MILP2~\cite{MaPe14}) was proposed for feature subset selection in linear SVM classification. 
The recursive feature elimination was implemented using the \texttt{caret} package in the R programming language.  
The MILO problems were solved using the optimization software IBM ILOG CPLEX 20.1.0.0~\cite{ib22}, where the interior-point method was used to solve relaxed subproblems.
The big-$M$ values for RMILO-K were set as in Eq.~\eqref{eq:bigM2}. 
For a selected subset $\hat{S}$ of features, SVM classifiers were trained using the {\tt sklearn.svm.LinearSVC} function (MILO-L) and the {\tt sklearn.svm.SVC} function (RFE-K and RMILO-K) in the Python programming language. 
We set the misclassification penalty parameter as $C=1$, which performed well for our synthetic datasets.  
We also used the scaling parameter $\gamma = \hat{\gamma}$, which was tuned by Eq.~\eqref{eq:sigest} for the subset-based Gaussian kernel function~\eqref{eq:sgk}. 

\subsection{Results for synthetic datasets}\label{sec4d}

Tables~\ref{tab:Syn_25_3}--\ref{tab:Syn_100_5} show the computational results of the three methods for feature subset selection for the synthetic datasets. 
Recall that the tables show average values over 10 repetitions, with standard errors of the ClsAcc and SetF1 values in parentheses, where the best ClsAcc and SetF1 values for each problem instance $(n,p,\mathrm{exp},\theta^*)$ are given in bold. 
Note also that where the tables show ``$>$10000.0'' in the column labeled ``Time,'' the computation reached the time limit of 10000\,s at least once out of 10~repetitions.

Table~\ref{tab:Syn_25_3} gives the results for the expansion factor $\mathrm{exp}=25$ and the subset size $\theta=\theta^*=3$. 
When $n=50$, our kernel-based MILO method (RMILO-K) achieved good accuracy for both classification (ClsAcc) and subset selection (SetF1). 
When $n=100$, the kernel-based recursive feature elimination (RFE-K) performed relatively well. 
On the whole, the linear-SVM-based MILO method (MILO-L) performed the worst. 

Table~\ref{tab:Syn_25_5} gives the results for the expansion factor $\mathrm{exp}=25$ and the subset size $\theta=\theta^*=5$. 
When $n=50$, RMILO-K maintained good accuracy for both classification and subset selection. 
When $n=100$, RFE-K and MILO-L had the best accuracy for classification and subset selection, respectively. 
However, RMILO-K selected fewer than half the features selected by MILO-L. 
Accordingly, it is also the case that RMILO-K delivered overall good performance with relatively few features. 

Table~\ref{tab:Syn_100_3} gives the results for the expansion factor $\mathrm{exp}=100$ and the subset size $\theta=\theta^*=3$. 
In this case, MILO-L outperformed the other kernel-based methods in terms of the classification accuracy. 
In other words, this dataset was compatible with linear SVM classifiers. 
On the other hand, the accuracy for classification and subset selection was higher for RMILO-K than for RFE-K overall.  

Table~\ref{tab:Syn_100_5} gives the results for the expansion factor $\mathrm{exp}=100$ and the subset size $\theta=\theta^*=5$. 
In this case, RMILO-K and RFE-K attained good classification accuracy when $n=50$ and $n=100$, respectively. 
As for the subset selection accuracy, although MILO-L had the overall best performance, RMILO-K with fewer features outperformed RFE-K on the whole. 

These results show that our MILO formulation delivers good prediction performance, especially when there are relatively few data instances. 
One of the main reasons for this is that the kernel--target alignment, which is the distance between the centroids of two response classes, is a performance measure that is robust against small datasets. 
Also, our MILO formulation can outperform recursive feature elimination in terms of the subset selection accuracy. 

\section{Conclusion}\label{sec5}
This paper dealt with feature subset selection for nonlinear kernel SVM classification. 
First, we introduced the INLO formulation for computing the best subset of features based on the kernel--target alignment, which is the distance between the centroids of two response classes in a high-dimensional feature space. 
Next, we reformulated the problem as a MILO problem and then devised some problem reduction techniques to solve the problem more efficiently. 

In computational experiments conducted using real-world and synthetic datasets, our MILO problems were solved in much shorter times than was the original INLO problem, and the computational efficiency was improved by our reduced MILO formulation. 
Our method often attained better classification accuracy than did the linear-SVM-based MILO formulation~\cite{MaPe14} and recursive feature elimination~\cite{GuGu08}, especially when there were relatively few data instances. 

It is known that feature subset selection for maximizing the kernel--target alignment leads to nonconvex optimization~\cite{NeSc05}. 
To our knowledge, we are the first to transform this subset selection problem into a MILO problem, which can be solved to optimality using optimization software. 
Note that if we try to solve the original nonconvex optimization problem exactly, then we cannot avoid numerical errors caused by its nonlinear objective function. 
In contrast, our method offers globally optimal solutions to small-sized problems without such numerical errors, and the obtained optimal solutions can be used to evaluate the solution quality of other algorithms. 
We also expect our formulation techniques to be applicable to other nonconvex optimization problems whose structures are similar to that of our problem. 

A future direction of study will be to develop an efficient algorithm specialized for directly solving our INLO problem~\eqref{obj:INLO}--\eqref{con2:INLO}. 
Another direction of future research will be to devise MIO formulations for feature subset selection using other kernel functions or other performance measures of kernel SVM classifiers. 

\section*{Acknowledgements}
This work was partially supported by JSPS KAKENHI Grant Numbers JP21K04526 and JP21K04527.

\clearpage

\begin{table}[!t]
\footnotesize
\begin{center}
\caption{Results for Hepatitis dataset: $(n,p)=(138,15)$}
\begin{tabular}{rrlrrrr}\toprule
$\theta$ & $\beta$ & Method & ObjVal & OptGap & $\lvert \hat{S} \rvert$ & Time \\ \midrule
3 & 0.25  & INLO-K & 0.180  & 151.1\% & 2 & $>$10000.0 \\
 &  & MILO-K & 0.236  & 0.0\% & 3 & 9498.3  \\
 &  & RMILO-K & 0.236  & 0.0\% & 3 & 716.6  \\ \cmidrule{2-7}
 & 1.00  & INLO-K & 0.389  & 190.6\% & 2 & $>$10000.0 \\
 &  & MILO-K & 0.389  & 171.6\% & 2 & $>$10000.0 \\
 &  & RMILO-K & 0.411  & 0.0\% & 3 & 2168.6  \\ \cmidrule{2-7}
 & 4.00  & INLO-K & 0.000  & $>$1000.0\% & 0 & $>$10000.0 \\
 &  & MILO-K & 0.489  & 0.0\% & 2 & 7127.1  \\
 &  & RMILO-K & 0.489  & 0.0\% & 2 & 1872.9  \\ \midrule
5 & 0.25  & INLO-K & 0.000  & 624.8\% & 0 & $>$10000.0 \\
 &  & MILO-K & 0.215  & 0.0\% & 5 & 7787.1  \\
 &  & RMILO-K & 0.215  & 0.0\% & 5 & 760.5  \\ \cmidrule{2-7}
 & 1.00  & INLO-K & 0.344  & 165.7\% & 4 & $>$10000.0 \\
 &  & MILO-K & 0.360  & 229.3\% & 5 & $>$10000.0 \\
 &  & RMILO-K & 0.392  & 0.0\% & 4 & 5135.0  \\ \cmidrule{2-7}
 & 4.00  & INLO-K & 0.406  & 278.8\% & 1 & $>$10000.0 \\
 &  & MILO-K & 0.463  & 162.7\% & 3 & $>$10000.0 \\
 &  & RMILO-K & 0.463  & 0.0\% & 2 & 2476.5  \\ \bottomrule
\end{tabular}
\label{tab:RW_Hep}
\end{center}
\end{table}

\begin{table}[!t]
\footnotesize
\begin{center}
\caption{Results for Zoo dataset: $(n,p)=(101,16)$}
\begin{tabular}{rrlrrrr}\toprule
$\theta$ & $\beta$ & Method & ObjVal & OptGap & $\lvert \hat{S} \rvert$ & Time \\ \midrule
3 & 0.25  & INLO-K & 0.303  & 0.0\% & 3 & 1624.7  \\
 &  & MILO-K & 0.303  & 0.0\% & 3 & 342.7  \\
 &  & RMILO-K & 0.303  & 0.0\% & 3 & 55.5  \\ \cmidrule{2-7}
 & 1.00  & INLO-K & 0.916  & 0.0\% & 3 & 1690.1  \\
 &  & MILO-K & 0.916  & 0.0\% & 3 & 388.4  \\
 &  & RMILO-K & 0.916  & 0.0\% & 3 & 81.5  \\ \cmidrule{2-7}
 & 4.00  & INLO-K & 1.445  & 0.0\% & 2 & 9489.9  \\
 &  & MILO-K & 1.445  & 0.0\% & 2 & 315.3  \\
 &  & RMILO-K & 1.445  & 0.0\% & 2 & 29.7  \\ \midrule
5 & 0.25  & INLO-K & 0.278  & 1.6\% & 5 & $>$10000.0 \\
 &  & MILO-K & 0.278  & 0.0\% & 5 & 379.1  \\
 &  & RMILO-K & 0.278  & 0.0\% & 5 & 89.1  \\ \cmidrule{2-7}
 & 1.00  & INLO-K & 0.657  & 32.5\% & 5 & $>$10000.0 \\
 &  & MILO-K & 0.726  & 0.0\% & 5 & 852.9  \\
 &  & RMILO-K & 0.726  & 0.0\% & 5 & 291.2  \\ \cmidrule{2-7}
 & 4.00  & INLO-K & 1.333  & 8.7\% & 3 & $>$10000.0 \\
 &  & MILO-K & 1.333  & 0.0\% & 3 & 503.3  \\
 &  & RMILO-K & 1.333  & 0.0\% & 3 & 62.1  \\ \bottomrule
\end{tabular}
\label{tab:RW_Zoo}
\end{center}
\end{table}

\begin{table}[!t]
\footnotesize
\begin{center}
\caption{Results for Parkinsons dataset: $(n,p)=(195,22)$}
\begin{tabular}{rrlrrrr}\toprule
$\theta$ & $\beta$ & Method & ObjVal & OptGap & $\lvert \hat{S} \rvert$ & Time \\ \midrule
3 & 0.25  & INLO-K & 0.000  & $>$1000.0\% & 0 & $>$10000.0 \\
 &  & MILO-K & 0.000  & $>$1000.0\% & 0 & $>$10000.0 \\
 &  & RMILO-K & 0.284  & 81.2\% & 3 & $>$10000.0 \\ \cmidrule{2-7}
 & 1.00  & INLO-K & 0.000  & $>$1000.0\% & 0 & $>$10000.0 \\
 &  & MILO-K & 0.000  & $>$1000.0\% & 0 & $>$10000.0 \\
 &  & RMILO-K & 0.316  & 206.3\% & 2 & $>$10000.0 \\ \cmidrule{2-7}
 & 4.00  & INLO-K & 0.000  & $>$1000.0\% & 0 & $>$10000.0 \\
 &  & MILO-K & 0.154  & $>$1000.0\% & 3 & $>$10000.0 \\
 &  & RMILO-K & 0.000  & $>$1000.0\% & 0 & $>$10000.0 \\ \midrule
5 & 0.25  & INLO-K & 0.000  & $>$1000.0\% & 0 & $>$10000.0 \\
 &  & MILO-K & 0.127  & 965.5\% & 3 & $>$10000.0 \\
 &  & RMILO-K & 0.251  & 88.6\% & 5 & $>$10000.0 \\ \cmidrule{2-7}
 & 1.00  & INLO-K & 0.000  & $>$1000.0\% & 0 & $>$10000.0 \\
 &  & MILO-K & 0.160  & $>$1000.0\% & 5 & $>$10000.0 \\
 &  & RMILO-K & 0.276  & 268.9\% & 2 & $>$10000.0 \\ \cmidrule{2-7}
 & 4.00  & INLO-K & 0.000  & $>$1000.0\% & 0 & $>$10000.0 \\
 &  & MILO-K & 0.158  & $>$1000.0\% & 5 & $>$10000.0 \\
 &  & RMILO-K & 0.000  & $>$1000.0\% & 0 & $>$10000.0 \\ \bottomrule
\end{tabular}
\label{tab:RW_Park}
\end{center}
\end{table}

\begin{table}[!t]
\footnotesize
\begin{center}
\caption{Results for Soybean dataset: $(n,p)=(47,45)$}
\begin{tabular}{rrlrrrr}\toprule
$\theta$ & $\beta$ & Method & ObjVal & OptGap & $\lvert \hat{S} \rvert$ & Time \\ \midrule
3 & 0.25  & INLO-K & 0.316  & 20.1\% & 3 & $>$10000.0 \\
 &  & MILO-K & 0.316  & 22.7\% & 3 & $>$10000.0 \\
 &  & RMILO-K & 0.316  & 0.0\% & 3 & 346.0  \\ \cmidrule{2-7}
 & 1.00  & INLO-K & 0.137  & $>$1000.0\% & 3 & $>$10000.0 \\
 &  & MILO-K & 0.926  & 0.0\% & 3 & 6669.9  \\
 &  & RMILO-K & 0.926  & 0.0\% & 3 & 618.2  \\ \cmidrule{2-7}
 & 4.00  & INLO-K & 1.419  & 26.3\% & 3 & $>$10000.0 \\
 &  & MILO-K & 1.451  & 0.0\% & 3 & 7055.2  \\
 &  & RMILO-K & 1.451  & 0.0\% & 3 & 144.0  \\ \midrule
5 & 0.25  & INLO-K & 0.300  & 15.1\% & 5 & $>$10000.0 \\
 &  & MILO-K & 0.267  & 547.6\% & 5 & $>$10000.0 \\
 &  & RMILO-K & 0.300  & 0.0\% & 5 & 2302.6  \\ \cmidrule{2-7}
 & 1.00  & INLO-K & 0.711  & 60.6\% & 5 & $>$10000.0 \\
 &  & MILO-K & 0.737  & 167.1\% & 5 & $>$10000.0 \\
 &  & RMILO-K & 0.870  & 0.0\% & 5 & 5054.1  \\ \cmidrule{2-7}
 & 4.00  & INLO-K & 0.966  & 84.8\% & 4 & $>$10000.0 \\
 &  & MILO-K & 1.391  & 20.9\% & 4 & $>$10000.0 \\
 &  & RMILO-K & 1.391  & 0.0\% & 4 & 414.7  \\ \bottomrule
\end{tabular}
\label{tab:RW_Soy}
\end{center}
\end{table}

\clearpage 

\begin{table}[!t]
\footnotesize
\begin{center}
\caption{Results for the synthetic dataset ($\mathrm{exp}=25$ and $\theta=\theta^*=3$)}
\begin{tabular}{rrlrrrrr}\toprule
$n$ & $p$ & Method & \multicolumn{1}{l}{ClsAcc} & \multicolumn{1}{l}{SetF1} & OptGap & $\lvert \hat{S} \rvert$ & Time \\ \midrule
50 & 10 & MILO-L & 0.867~($\pm$0.007) & \textbf{0.833}~($\pm$0.056) & 0.0\% & 3.0  & 0.1  \\
 &  & RFE-K & 0.898~($\pm$0.019) & 0.667~($\pm$0.055) & \multicolumn{1}{c}{---} & 2.3  & 13.0  \\
 &  & RMILO-K & \textbf{0.935}~($\pm$0.005) & 0.820~($\pm$0.020) & 0.0\% & 2.1  & 46.8  \\ \cmidrule{2-8}
 & 20 & MILO-L & 0.849~($\pm$0.010) & 0.700~($\pm$0.078) & 0.0\% & 3.0  & 0.2  \\
 &  & RFE-K & 0.905~($\pm$0.014) & 0.690~($\pm$0.052) & \multicolumn{1}{c}{---} & 2.1  & 13.2  \\
 &  & RMILO-K & \textbf{0.931}~($\pm$0.004) & \textbf{0.807}~($\pm$0.025) & 0.0\% & 2.2  & 785.2  \\ \cmidrule{2-8}
 & 30 & MILO-L & 0.837~($\pm$0.011) & 0.633~($\pm$0.078) & 0.0\% & 3.0  & 0.4  \\
 &  & RFE-K & 0.849~($\pm$0.036) & 0.647~($\pm$0.062) & \multicolumn{1}{c}{---} & 1.8  & 13.4  \\
 &  & RMILO-K & \textbf{0.935}~($\pm$0.006) & \textbf{0.827}~($\pm$0.032) & 5.3\% & 2.3  & $>$10000.0 \\ \midrule
100 & 10 & MILO-L & 0.866~($\pm$0.006) & 0.833~($\pm$0.056) & 0.0\% & 3.0  & 0.1  \\
 &  & RFE-K & 0.\textbf{944}~($\pm$0.009) & \textbf{0.880}~($\pm$0.033) & \multicolumn{1}{c}{---} & 2.4  & 13.0  \\
 &  & RMILO-K & 0.934~($\pm$0.010) & 0.810~($\pm$0.043) & 0.0\% & 2.1  & 378.4  \\ \cmidrule{2-8}
 & 20 & MILO-L & 0.865~($\pm$0.006) & \textbf{0.833}~($\pm$0.056) & 0.0\% & 3.0  & 0.4  \\
 &  & RFE-K & 0.922~($\pm$0.018) & 0.830~($\pm$0.047) & \multicolumn{1}{c}{---} & 2.2  & 13.2  \\
 &  & RMILO-K & \textbf{0.933}~($\pm$0.010) & 0.810~($\pm$0.043) & 0.0\% & 2.1  & 4660.3  \\ \cmidrule{2-8}
 & 30 & MILO-L & 0.866~($\pm$0.006) & \textbf{0.833}~($\pm$0.056) & 0.0\% & 3.0  & 0.4  \\
 &  & RFE-K & \textbf{0.922}~($\pm$0.018) & 0.830~($\pm$0.047) & \multicolumn{1}{c}{---} & 2.2  & 13.4  \\
 &  & RMILO-K & 0.906~($\pm$0.011) & 0.650~($\pm$0.050) & 130.9\% & 1.5  & $>$10000.0 \\ \bottomrule
\end{tabular}
\label{tab:Syn_25_3}
\end{center}
\end{table}

\begin{table}[!t]
\footnotesize
\begin{center}
\caption{Results for the synthetic dataset ($\mathrm{exp}=25$ and $\theta=\theta^*=5$)}
\begin{tabular}{rrlrrrrr}\toprule
$n$ & $p$ & Method & \multicolumn{1}{l}{ClsAcc} & \multicolumn{1}{l}{SetF1} & OptGap & $\lvert \hat{S} \rvert$ & Time \\ \midrule
50 & 10 & MILO-L & 0.872~($\pm$0.008) & \textbf{0.720}~($\pm$0.044) & 0.0\% & 5.0  & $<$0.1 \\
 &  & RFE-K & \textbf{0.885}~($\pm$0.008) & 0.618~($\pm$0.069) & \multicolumn{1}{c}{---} & 2.7  & 18.9  \\
 &  & RMILO-K & 0.875~($\pm$0.011) & 0.661~($\pm$0.030) & 0.0\% & 2.5  & 72.3  \\ \cmidrule{2-8}
 & 20 & MILO-L & 0.854~($\pm$0.008) & 0.640~($\pm$0.040) & 0.0\% & 5.0  & 0.2  \\
 &  & RFE-K & 0.860~($\pm$0.007) & 0.466~($\pm$0.048) & \multicolumn{1}{c}{---} & 2.0  & 19.1  \\
 &  & RMILO-K & \textbf{0.873}~($\pm$0.011) & \textbf{0.661}~($\pm$0.030) & 0.0\% & 2.5  & 3472.5  \\ \cmidrule{2-8}
 & 30 & MILO-L & 0.849~($\pm$0.007) & 0.600~($\pm$0.030) & 0.0\% & 5.0  & 0.7  \\
 &  & RFE-K & 0.860~($\pm$0.007) & 0.466~($\pm$0.048) & \multicolumn{1}{c}{---} & 2.0  & 19.3  \\
 &  & RMILO-K & \textbf{0.871}~($\pm$0.011) & \textbf{0.643}~($\pm$0.029) & 74.6\% & 2.4  & $>$10000.0 \\ \midrule
100 & 10 & MILO-L & 0.892~($\pm$0.004) & \textbf{0.880}~($\pm$0.033) & 0.0\% & 5.0  & 0.1  \\
 &  & RFE-K & \textbf{0.896}~($\pm$0.010) & 0.643~($\pm$0.062) & \multicolumn{1}{c}{---} & 2.8  & 18.9  \\
 &  & RMILO-K & 0.886~($\pm$0.010) & 0.643~($\pm$0.029) & 0.0\% & 2.4  & 524.5  \\ \cmidrule{2-8}
 & 20 & MILO-L & 0.880~($\pm$0.006) & \textbf{0.780}~($\pm$0.047) & 0.0\% & 5.0  & 0.4  \\
 &  & RFE-K & \textbf{0.890}~($\pm$0.010) & 0.557~($\pm$0.063) & \multicolumn{1}{c}{---} & 2.2  & 19.1  \\
 &  & RMILO-K & 0.879~($\pm$0.009) & 0.625~($\pm$0.027) & 72.0\% & 2.3  & $>$10000.0 \\ \cmidrule{2-8}
 & 30 & MILO-L & 0.872~($\pm$0.007) & \textbf{0.640}~($\pm$0.027) & 0.0\% & 5.0  & 1.2  \\
 &  & RFE-K & \textbf{0.887}~($\pm$0.009) & 0.538~($\pm$0.053) & \multicolumn{1}{c}{---} & 2.9  & 19.3  \\
 &  & RMILO-K & 0.864~($\pm$0.003) & 0.540~($\pm$0.024) & 187.5\% & 2.0  & $>$10000.0 \\ \bottomrule
\end{tabular}
\label{tab:Syn_25_5}
\end{center}
\end{table}

\clearpage

\begin{table}[!t]
\footnotesize
\begin{center}
\caption{Results for the synthetic dataset ($\mathrm{exp}=100$ and $\theta=\theta^*=3$)}
\begin{tabular}{rrlrrrrr}\toprule
$n$ & $p$ & Method & \multicolumn{1}{l}{ClsAcc} & \multicolumn{1}{l}{SetF1} & OptGap & $\lvert \hat{S} \rvert$ & Time \\ \midrule
50 & 10 & MILO-L & \textbf{0.800}~($\pm$0.012) & 0.833~($\pm$0.056) & 0.0\% & 3.0  & 0.1  \\
 &  & RFE-K & 0.720~($\pm$0.019) & 0.687~($\pm$0.044) & \multicolumn{1}{c}{---} & 2.6  & 13.1  \\
 &  & RMILO-K & 0.757~($\pm$0.025) & \textbf{0.853}~($\pm$0.066) & 0.0\% & 2.6  & 88.3  \\ \cmidrule{2-8}
 & 20 & MILO-L & \textbf{0.774}~($\pm$0.023) & 0.767~($\pm$0.071) & 0.0\% & 3.0  & 0.2  \\
 &  & RFE-K & 0.707~($\pm$0.022) & 0.640~($\pm$0.055) & \multicolumn{1}{c}{---} & 2.7  & 13.3  \\
 &  & RMILO-K & 0.746~($\pm$0.025) & \textbf{0.820}~($\pm$0.066) & 0.0\% & 2.6  & 1807.0  \\ \cmidrule{2-8}
 & 30 & MILO-L & \textbf{0.757}~($\pm$0.021) & \textbf{0.700}~($\pm$0.060) & 0.0\% & 3.0  & 0.4  \\
 &  & RFE-K & 0.699~($\pm$0.023) & 0.637~($\pm$0.059) & \multicolumn{1}{c}{---} & 2.4  & 13.5  \\
 &  & RMILO-K & 0.699~($\pm$0.020) & 0.653~($\pm$0.057) & 226.9\% & 2.6  & $>$10000.0 \\ \midrule
100 & 10 & MILO-L & \textbf{0.816}~($\pm$0.008) & 0.867~($\pm$0.054) & 0.0\% & 3.0  & 0.1  \\
 &  & RFE-K & 0.793~($\pm$0.009) & 0.820~($\pm$0.043) & \multicolumn{1}{c}{---} & 2.6  & 13.1  \\
 &  & RMILO-K & 0.808~($\pm$0.006) & \textbf{0.920}~($\pm$0.033) & 0.0\% & 2.6  & 645.2  \\ \cmidrule{2-8}
 & 20 & MILO-L & \textbf{0.810}~($\pm$0.008) & 0.833~($\pm$0.056) & 0.0\% & 3.0  & 0.5  \\
 &  & RFE-K & 0.785~($\pm$0.008) & 0.780~($\pm$0.032) & \multicolumn{1}{c}{---} & 2.4  & 13.3  \\
 &  & RMILO-K & 0.802~($\pm$0.008) & \textbf{0.887}~($\pm$0.040) & 134.8\% & 2.6  & $>$10000.0 \\ \cmidrule{2-8}
 & 30 & MILO-L & \textbf{0.807}~($\pm$0.008) & \textbf{0.800}~($\pm$0.054) & 0.0\% & 3.0  & 0.3  \\
 &  & RFE-K & 0.774~($\pm$0.008) & 0.733~($\pm$0.022) & \multicolumn{1}{c}{---} & 2.5  & 13.5  \\
 &  & RMILO-K & 0.726~($\pm$0.028) & 0.687~($\pm$0.072) & 709.0\% & 2.1  & $>$10000.0 \\ \bottomrule
\end{tabular}
\label{tab:Syn_100_3}
\end{center}
\end{table}

\begin{table}[!t]
\footnotesize
\begin{center}
\caption{Results for the synthetic dataset ($\mathrm{exp}=100$ and $\theta=\theta^*=5$)}
\begin{tabular}{rrlrrrrr}\toprule
$n$ & $p$ & Method & \multicolumn{1}{l}{ClsAcc} & \multicolumn{1}{l}{SetF1} & OptGap & $\lvert \hat{S} \rvert$ & Time \\ \midrule
50 & 10 & MILO-L & 0.793~($\pm$0.008) & \textbf{0.700}~($\pm$0.061) & 0.0\% & 5.0  & $<$0.1 \\
 &  & RFE-K & 0.809~($\pm$0.007) & 0.535~($\pm$0.047) & \multicolumn{1}{c}{---} & 3.0  & 19.2  \\
 &  & RMILO-K & \textbf{0.822}~($\pm$0.005) & 0.582~($\pm$0.020) & 0.0\% & 2.2  & 138.7  \\ \cmidrule{2-8}
 & 20 & MILO-L & 0.780~($\pm$0.007) & \textbf{0.620}~($\pm$0.070) & 0.0\% & 5.0  & 0.2  \\
 &  & RFE-K & 0.805~($\pm$0.008) & 0.498~($\pm$0.040) & \multicolumn{1}{c}{---} & 2.6  & 19.4  \\
 &  & RMILO-K & \textbf{0.820}~($\pm$0.006) & 0.575~($\pm$0.022) & 178.3\% & 2.3  & $>$10000.0 \\ \cmidrule{2-8}
 & 30 & MILO-L & 0.770~($\pm$0.008) & 0.520~($\pm$0.053) & 0.0\% & 5.0  & 1.1  \\
 &  & RFE-K & 0.810~($\pm$0.009) & 0.513~($\pm$0.038) & \multicolumn{1}{c}{---} & 2.3  & 19.3  \\
 &  & RMILO-K & \textbf{0.822}~($\pm$0.005) & \textbf{0.557}~($\pm$0.010) & 295.3\% & 2.2  & $>$10000.0 \\ \midrule
100 & 10 & MILO-L & 0.808~($\pm$0.004) & \textbf{0.740}~($\pm$0.043) & 0.0\% & 5.0  & 0.1  \\
 &  & RFE-K & \textbf{0.836}~($\pm$0.009) & 0.639~($\pm$0.049) & \multicolumn{1}{c}{---} & 3.0  & 19.2  \\
 &  & RMILO-K & 0.827~($\pm$0.007) & 0.571~($\pm$0.000) & 0.0\% & 2.0  & 708.6  \\ \cmidrule{2-8}
 & 20 & MILO-L & 0.808~($\pm$0.005) & \textbf{0.720}~($\pm$0.053) & 0.0\% & 5.0  & 0.6  \\
 &  & RFE-K & \textbf{0.836}~($\pm$0.007) & 0.552~($\pm$0.028) & \multicolumn{1}{c}{---} & 3.0  & 19.5  \\
 &  & RMILO-K & 0.826~($\pm$0.007) & 0.571~($\pm$0.000) & 195.2\% & 2.0  & $>$10000.0 \\ \cmidrule{2-8}
 & 30 & MILO-L & 0.804~($\pm$0.005) & 0.540~($\pm$0.052) & 0.0\% & 5.0  & 0.9  \\
 &  & RFE-K & \textbf{0.830}~($\pm$0.007) & 0.518~($\pm$0.030) & \multicolumn{1}{c}{---} & 3.6  & 19.3  \\
 &  & RMILO-K & 0.826~($\pm$0.007) & \textbf{0.571}~($\pm$0.000) & 336.5\% & 2.0  & $>$10000.0 \\ \bottomrule
\end{tabular}
\label{tab:Syn_100_5}
\end{center}
\end{table}

\clearpage

\bibliography{sn-bibliography}

\end{document}